\documentclass{llncs}
\usepackage{comment,amssymb,amsmath,tikz,blkarray}
\usepackage{graphicx}
\usepackage{caption,subcaption}
\long\def\BEGINCOMMENT #1\ENDCOMMENT{\relax}

\newcommand{\maps}\longrightarrow
\newcommand{\cmaps}\Longrightarrow

\newcommand{\dest}{\operatorname{{\bf D}}}

\title{Conjunctions of Among Constraints}

\author{V\'{\i}ctor Dalmau}

\institute{Dept. of Information and Communication Technologies, Universitat Pompeu Fabra}

\begin{document}

\maketitle

\newtheorem{myclaim}{Claim}
\newtheorem{observation}{Observation}

\renewcommand{\max}{max}
\renewcommand{\min}{min}

\newcommand{\fpt}{\operatorname{FPT}}
\newcommand{\wone}{\operatorname{W[1]}}
\newcommand{\plhom}{\textsc{p-LHom}}
\newcommand{\gaifman}{\operatorname{Gaifman}}
\newcommand{\boolean}{B}
\newcommand{\oneinthree}{\textsc{One-In-Three Sat}}

\newcommand{\filter}{\textsc{DomFilter}}
\newcommand{\Constraint}[1]{\textsc{#1}}
\newcommand{\alldifferent}{\Constraint{AllDiff}}
\newcommand{\globalcardinality}{\Constraint{GCC}}
\newcommand{\among}{\Constraint{Among}}
\newcommand{\sequence}{\Constraint{Sequence}}
\newcommand{\TFO}{\Constraint{TFO}}
\newcommand{\threeFO}{\Constraint{3FO}}
\newcommand{\ordereddistribute}{\Constraint{OrderedDistribute}}
\newcommand{\symgcc}{\Constraint{Symmetric-GCC}}
\newcommand{\globalsequencing}{\Constraint{GlobalSequencing}}
\newcommand{\cardinalitymatrix}{\Constraint{CardinalityMatrix}}
\newcommand{\cac}{\Constraint{CAC}}

\newcommand{\range}{\operatorname{Range}}
\newcommand{\scope}{\operatorname{Scope}}


\begin{abstract}
 Many existing global constraints can be encoded as a conjunction
of among constraints. An among constraint holds if the number of the variables in its scope
whose value belongs to a prespecified set, which we call its range, is within some given bounds. It is known that domain
filtering algorithms can benefit from reasoning about the interaction of among constraints
so that values can be filtered out taking into consideration several among constraints simultaneously. 
The present paper embarks into a systematic investigation on the circumstances under which it
is possible to obtain efficient and complete domain filtering algorithms for conjunctions of among 
constraints. We start by observing that restrictions on both the scope and the range of the 
among constraints are necessary to obtain meaningful results. Then, we derive 
a domain flow-based filtering algorithm and present several applications. In particular,  it is
shown that the algorithm unifies and generalizes several previous existing results.

\end{abstract}

\section{Introduction}

Global constraints play a major role in constraint programming. Very informally, a global constraint
is a constraint, or perhaps more precisely, a family of constraints, which is versatile enough to be able to express restrictions 
that are encountered often in practice.  For example, one of the most widely used global constraints 
is the 'All different' constraint, \alldifferent$(S)$ where $S=\{x_1,\dots,x_n\}$ is a set of variables, which 
specifies that the values assigned to the variables in $S$ must be all pairwise different. This sort of
restriction arises naturally in many areas, such as for example scheduling problems, where the variables $x_1,\dots,x_n$
could represent $n$ activities that must be assigned different times of a common resource.

Besides is usefulness in simplifying the modeling or programming task, 
global constraints also improve greatly the efficiently of propagation-search based solvers. This type of solver
performs a tree search that constructs partial assignments and enforces some sort of propagation or local consistency that prunes the
space search. Different forms of consistency, including (singleton) bounds consistency, (singleton, generalized) arc-consistency, path consistency and
many others, can be used in the propagation phase. One of the most commonly 
used forms of local consistency is {\em domain consistency}, also called generalized arc-consistency. A
domain consistency algorithm keeps, for every variable $v$, a list, $L(v)$, of feasible values, which is updated, by removing
a value $d$ from it, when some constraint in the problem guarantees that $v$ cannot take value $d$ in any solution.
One of the key reasons of the success of global constraints is that they enable the use of efficient filtering algorithms specifically tailored for 
them. 

Several constraints studied in the literature including \alldifferent{}\cite{Lauriere78}, \globalcardinality \cite{Oplobedu78}, 
\symgcc{} \cite{Kocjan04}, \sequence{} \cite{Beldiceanu94}, \globalsequencing{} \cite{Regin97}, \ordereddistribute{} \cite{Petit11}, and \cardinalitymatrix{} 
\cite{Regin04}
can be decomposed as the conjunction of a simpler family of constraints, called among constraints \cite{Beldiceanu94}. An among constraint has the form 
$\among(S,R,\min,\max)$ where $S$ is again a set of variables called the {\em scope}, 
$R$ is a subset of the possible values, called the {\em range}, and $\min,\max$ are integers
This constraint 
specifies that the number of variables in $S$ that take a value in $R$ 
must be in the range $\{\min,\dots,\max\}$.  For example, the constraint \alldifferent$(S)$ can be expressed as the
conjunction of constraints $\among(S,\{d_1\},0,1),\dots,\among(S,\{d_k\},0,1)$ where $d_1,\dots,d_k$ are the set 
of all feasible values for the variables in $S$.

Besides encoding more complex global cardinality constraints, conjunctions of among constraints, (\cac{}), appear in many problems, such 
as Sudoku or latin squares. In consequence, \cac s have been previously studied \cite{Chabert12,Regin05,Hoeve09}, specially the particular case of 
conjunctions of \alldifferent{} constraints \cite{Appa04,Appa05,Bessiere10,Fellows13,Kutz08,Lardeux08,Magos12}.
Although deciding the satisfiability of an arbitrary conjunction of among constraints is NP-complete \cite{Regin05} this 
body of work shows that sometimes there are benefits in reasoning about the interaction between the among constraints.
Hence, it is important to understand under which circumstances among constraints can be combined in order to endow CSP solvers
with the ability to propagate taking into consideration several among constraints simultaneously. 
The aim of the present paper is to contribute to this line of research. To this end we first observe
that restrictions on {\em both} the scope and the range of the among constraints are necessary to obtain meaningful results.
Then we embark in a systematic study of which such restrictions guarantee efficient propagation algorithms. In particular,
we introduce a general condition such that every \cac{} satisfying it 
admits an efficient and complete domain filtering algorithm. This condition
basically expresses that the matrix of a system of linear equations encoding the \cac{} instance belongs to a particular class
of totally unimodular matrices known as network matrices. This allows to reformulate the domain
filtering problem in terms of flows in a network graph and apply the methodology derived by R\'egin \cite{Regin94,Regin96}.
The algorithm thus obtained, although simple, unifies and generalizes existing domain filtering algorithms for several global 
constraints,
including \alldifferent{}, \globalcardinality{}, \sequence{}, \symgcc{}, \ordereddistribute{} as well as for
other problems expressed as conjunctions of among constraints  in \cite{RazgonOP07,Regin05}.
A nice feature of our approach is that it 
abstracts out the construction of the network flow problem, so that when exploring a new \cac{}
one might leave out the usually messy details of the design of the network graph and reason purely in combinatorial terms. 

Several filtering methods have been obtained by decomposing a global constraint
into a combination of among constraints. For example the first polynomial-time filtering
algorithm for the \sequence{} constraint \cite{Hoeve09} is obtained explicitely in this way.
However, there have been very few attempts to determine sistematically which particular conjunctions
of among constraints allow efficient filtering algorithms. The seminal paper in this direction is \cite{Regin96} which identifies
several combinations of among and \globalcardinality{} constraints that admit a complete and efficient domain filtering algorithm (see Section \ref{sec:applications} for more details). Our approach is more general as it subsumes the tractable cases introduced in \cite{Regin96}. Another closely related work is \cite{RazgonOP07}
where two tractable combinations of boolean \cac{}s, called \TFO{} and \threeFO{}, are identified. The approach in \cite{RazgonOP07} differs from ours
in two aspects: it deals with optimization problem and also considers restricions on the $\min$ and $\max$ parameters of the among constraints while we only consider
restriction on the scope and range. A different family of \cac{}s has been investigated in \cite{Chabert12} although the work in \cite{Chabert12} focuses in bound consistency instead of domain consistency. 

Other approaches to the design of filtering algorithms for combinations (but not necessarily conjunctions) of global (but not necessarily among) constraints are described in \cite{BacchusW05,BessiereHHKW08,BessiereHHKW09}. The method introduced in \cite{BacchusW05} deals with logical combinations of some primitive constraints but differs substantially from ours in the sense that it cannot capture a single among constraint. The work reported in \cite{BessiereHHKW08,BessiereHHKW09} does not guarantee tractability. 

Several proofs are omiitted due to space restrictions. They can be found in the Appendix.


\section{Preliminaries}\
A 
{\em conjunction of among constraints}, (\cac{}) is a tuple $(V,D,L,{\mathcal C})$ where $V$ is a finite set whose elements are called {\em variables},
$D$ is a finite set called {\em domain}, $L:V\rightarrow 2^{D}$ is a mapping that sends every variable $v$ to a subset of $D$, which we call its {\em list}, and
${\mathcal C}$ is a finite set of {\em constraints} where a constraint is an expression of the form
$\among(S,R,\min,\max)$ where $S\subseteq V$ is called the {\em scope} of the constraint,
$R\subseteq D$ is called {\em range} of the constraint, and $\min,\max$ are integers satisfying $0\leq \min\leq\max\leq |S|$.

A {\em solution} of $(V,D,L,{\mathcal C})$ is a mapping $s:V\rightarrow D$ such that $s(v)\in L(v)$ for every variable $v\in V$ and
$\min\leq|\{v\in S \mid s(v)\in R\}|\leq\max$ for every constraint $\among(S,R,\min,\max)$ in ${\mathcal C}$.

\begin{example}(\globalcardinality{} and \alldifferent{} constraints)
\label{ex:gcc}
The global cardinality constraint\footnote{We want to stress here that a global constraint is not a single
constraint but, in fact, a family of them.}, \globalcardinality{} \cite{Oplobedu78} corresponds to 
instances $(V,D,L,{\mathcal C})$ where all the constraints have the form
$\among(V,\{d\},\min,\max)$ with $d\in D$.
The \alldifferent{} constraint is the particular case obtained when, additionally, $\min=0$ and $\max=1$. 
\end{example}

Let $I=(V,D,L,{\mathcal C})$ be a \cac{}. We say that a value $d\in D$ is {\em supported} for a variable $v\in V$ if
there is a solution $s$ of $I$ with $s(v)=d$. In this paper we focus in the following computational problem, which we will call {\em domain filtering}: given a 
\cac{}, compute the set of all the non supported values for each of its variables.

This definition is motivated by the following scenario: think of $(V,D,L,{\mathcal C})$ 
as defining a constraint which is part of a CSP instance that is being solved by a search-propagation algorithm that
enforces domain consistency. Assume that at any stage of the execution of the algorithm, $L$ encodes the actual feasible values for each variable
in $V$. Then, the domain filtering problem is basically the task of identifying all
the values that need to be pruned by considering the constraint encoded by $(V,D,L,{\mathcal C})$.



\section{Network hypergraphs}

An {\em hypergraph} $H$ is a tuple, $(V(H),E(H))$, where $V(H)$ is a finite set whose elements are called {\em nodes} and $E(H)$
is set whose elements are subsets of $V(H)$, called hyperedges. 
An hypergraph is {\em totally unimodular} if its incidence matrix $M$ is totally unimodular, that is, if every square submatrix of $M$ has determinant $0$, $+1$, or $-1$. 
In this paper we are concerned with a subset of totally unimodular hypergraphs called network hypergraphs. In order to 
define network hypergraph we need to introduce a few definitions.

An {\em oriented tree} $T$ is any directed tree obtained by orienting the edges of an undirected tree. A {\em path} $p$ in
$T$ is any sequence $x_1,e_1,x_2,\dots,e_{n-1},x_n$ where $x_1,\dots,x_n$ are different vertices of $T$,
$e_1,\dots,e_{n-1}$ are edges in $T$ and for every $1\leq i<n$, either $e_i=(x_i,x_{i+1})$ 
or $e_i=(x_{i+1},x_i)$. The {\em polarity} of an edge $e\in E(T)$ wrt. $p$ is defined to
be $+1$ (or positive) if $e=(x_i,x_{i+1})$ for some $1\leq i<n$, $-1$ (or negative) if
$e=(x_{i+1},x_i)$ for some $1\leq i< n$, and $0$ if $e$ does not appear in $p$.
A path $p$ has positive (resp. negative) polarity if
all its edges have positive (resp. negative) polarity. Paths with positive polarity are also called directed paths.
Since an oriented tree does not contain symmetric edges, we might represent a path by giving only its sequence
of nodes $x_1,\dots,x_n$.


We say that an oriented tree $T$ {\em defines} an hypergraph $H$ if we can associate to
every hyperedge $h\in E(H)$ an edge $e_h\in E(T)$ and to every node $v\in V(H)$ a directed path $p_v$ in $T$ such that
for every $v\in V(H)$ and $h\in E(V)$, $v\in h$ if and only if $e_h$ belongs to $p_v$. We say that an $H$ is a {\em network}
hypergraph if there is an oriented tree that defines it.

\begin{example}
\label{ex:hypergraph}
The hypergraph $H$ with variable-set $\{v_1,\dots,v_6\}$ and hyperedge-set $\{h_1,\dots,h_5\}$
given in Figure \ref{fig:1a} is a network hypergraph as it is defined by tree $T$ given in Figure \ref{fig:1b} where we have indicated, using labels on the edges, 
the edge in $T$ associated to every hyperedge in $H$. We  associate 
to every variable $v_i$, $1\leq i\leq 6$ the directed path $s_{(i-1\mod 2)},r,t_{(i-1\mod 3)}$ in $T$. It can be readily checked
that under this assignment $T$ defines $H$.

\captionsetup[subfigure]{labelformat=empty}

\begin{figure}
\centering
	\begin{subfigure}[t]{1.3in}
		\centering
\tikzstyle{vertex}=[circle,fill=black!25,minimum size=12pt,inner sep=0pt]
\tikzstyle{edge} = [opacity=.5,fill opacity=.5,line cap=round, line join=round, line width=1pt]
\pgfdeclarelayer{background}
\pgfsetlayers{background,main}
\begin{tikzpicture}[scale=0.65,auto]
\foreach \pos/\name in {{(0,0)/{v_1}},{(1,0)/{v_3}},{(2,0)/{v_5}},{(0,1)/{v_4}},{(1,1)/{v_6}},{(2,1)/{v_2}}}
        \node[vertex] (\name) at \pos {$\name$};

   \draw (1,0) ellipse (2cm and 0.4cm);
    \draw (1,1) ellipse (2cm and 0.4cm);
   \draw (0,0.5) ellipse (0.4cm and 1.4cm);
   \draw (1,0.5) ellipse (0.4cm and 1.4cm);
   \draw (2,0.5) ellipse (0.4cm and 1.4cm);
\node at (-1.4,0) {$h_4$};
\node at (-1.4,1) {$h_5$};
\node at (0,2.2) {$h_1$};
\node at (1,2.2) {$h_3$};
\node at (2,2.2) {$h_2$}; 
\end{tikzpicture}
		\caption{Fig.\,{\ref{fig:1a}}:\,\,{Hypergraph} $H$}\label{fig:1a}		
	\end{subfigure}
	\quad \quad \quad \quad \quad \quad \quad
	\begin{subfigure}[t]{1.4in}
		\centering
\tikzstyle{vertex}=[circle,fill=black!25,minimum size=12pt,inner sep=0pt]
\tikzstyle{edge} = [draw,thick,-]
\tikzstyle{diedge} = [draw,thick,->]
\tikzstyle{weight} = [font=\small]
\begin{tikzpicture}[scale=0.65,auto]
\foreach \pos/\name in {{(0,0)/r},{(-2,1)/{s_0}},{(-2,-1)/{s_1}},{(2,1.5)/{t_0}},{(2,0)/{t_1}},{(2,-1.5)/{t_2}}}
        \node[vertex] (\name) at \pos {$\name$};
    \path[diedge] ({s_0}) -- node[pos=0.2,style=sloped] {$h_4$} (r);
    \path[diedge] ({s_1}) -- node[pos=0.8,style=sloped] {$h_5$} (r);
    \path[diedge] (r) -- node[pos=0.7,style=sloped] {$h_1$} ({t_0});
    \path[diedge] (r) -- node[pos=0.6,style=sloped] {$h_2$} ({t_1});
    \path[diedge] (r) -- node[pos=0.3,style=sloped] {$h_3$} ({t_2});
\end{tikzpicture}
		\caption{Fig.\,{\ref{fig:1b}}:\,\,{Oriented tree} $T$}\label{fig:1b}
	\end{subfigure}
\end{figure}

\end{example}

Sometimes, it will be convenient to assume that the tree $T$ defining $H$ is minimal in the sense
that no tree with fewer nodes defines $H$. Minimal trees have the nice property that every edge $e$
in $T$ is associated with some hyperedge of $H$. Indeed, assume that some edge $e=(x,y)$ 
is not associated to any hyperedge in $H$, then one could find an smaller tree $T$ defining
$H$ by {\em contracting} edge $e$, that is, by merging $x$ and $y$ into a new node $z$ that has as in-neighbours
the union of all in-neighbours of $x$ and $y$ and, as out-neighbours, the union of all out-neighbours of $x$ and $y$.

Since the vast majority of the trees defined in this paper will be oriented we shall usually
drop 'oriented'. So, unless, otherwise explicitly stated, a tree is always an oriented tree. Finally, we note that
one can decide whether a given hypergraph is a network hypergraph in time $O(e^3v^2)$
where $e$ is the number of hyperedges and $v$ is the number of nodes  (see chapter 20 in \cite{Schrijver98} for example).

\section{Restricting only the scope or the range}
It has been shown by R\'egin \cite{Regin05} that the domain filtering problem for \cac{s} is NP-hard. Still, efficient algorithms are known for some particular cases. It seems natural to start by asking which tractable subcases of the problem can be explained by considering {\em only} the scopes of the constraints.
This question has a close similarity to the study of the so-called structural restrictions
of the CSP (see, for example \cite{Gottlob00} for a survey) and, not surprisingly, it can be solved by applying results developed there. Indeed, it follows easily from a result of F\"arnquivst and Jonsson \cite{Farnqvist07} that, modulo some mild technical assumptions, if one allows arbitrary ranges in constraints, then the domain filtering problem is 
solvable in polynomial time if and only if the hypergraph of the scopes of the constraints has bounded tree-width (see Appendix A 
for precise statement and the proof). This result, although delineates exactly the border between tractability and intractability, turns
out to be not very useful in explaining the tractability of global constraints. This is due to the fact that global cardinality constraints 
defined by  conjunctions of among constraints usually have constraints with large scopes and the cardinality of the scope in a constraint  is a lower bound
on the tree-width of its scope hypergraph. 

One can also turn the attention to the range of constraints and inquiry whether there are tractable subcases of the problem that can be explained {\em only} 
by the range of the constraints. Here, again the response is not too useful. Indeed, it is very easy to show (see again 
Appendix A) 
that as soon
as we allow some non-trivial range $R$ (that is some range different than the empty set and than the whole domain) and arbitrary scopes in the
among constraints, then the domain filtering problem becomes NP-complete. 

In view of this state of affairs it is meaningful to consider families of conjunctions of among constraints that are obtained by 
restricting {\em simultaneously} the scope and the range of the constraints occurring in them. This is done in the next section.

\section{A flow-based algorithm}
\label{sec:alg}
Let $I=(V,D,L,{\mathcal C})$ be conjunction of among constraints. We will deal first with the case
in which $D$ is a boolean, say $D=\{0,1\}$. Hence, we
can assume that every constraint $\among(S,R,\min,\max)$ in ${\mathcal C}$ satisfies $R=\{1\}$ since, if
$R=\{0\}$ it can be reformulated as $\among(S,\{1\},|S|-\max,|S|-\min)$. We also assume that
$L(v)=\{0,1\}$ for every $v\in V$ since if $L(v)\neq\{0,1\}$ we could obtain easily an equivalent instance
without variable $v$.  

It is easy
to construct a system of linear equations whose feasible integer solutions encode the solutions of $I$. 
Let $v_1,\dots,v_n$ be the
variables of $I$ and let $C_j=(S_j,\{1\},\min_j,\max_j),j=1,\dots,m$,
be its constraints. The system
has variables 
$x_i (1\leq i\leq n)$, $y_j (1\leq j\leq m)$ and the following equations:
\begin{align*}
 & y_j+\sum_{v_i\in S_j} x_i=\max_j  &  j=1,\dots,m \label{eq1} \\
 & 0\leq y_j\leq \max_j-\min_j &  j=1,\dots,m  \\
 & 0\leq x_i\leq 1   &  i=1,\dots,n  
\end{align*}
which we express in matrix form as 
\begin{align*}
& Mz=a & \\
& 0\leq z\leq c
\end{align*}
with $z^T=(x_1,\dots,x_n,y_1,\dots,y_m)$ (see Example \ref{ex:I}).

If $M$ is totally unimodular then one can perform domain filtering in polynomial time. Indeed,
for every $v_i\in V$ and $d\in L(v)$, we might decide whether $d$ is a supported value for $v$ as follows:
add equation $x_i=d$ to the system 
and decide whether there exists a feasible solution of its 
linear relaxation using a LP solver. It follows from
total unimodularity (see Theorem 19.1 in \cite{Schrijver98} for example) that such a feasible solution exists if and only if $d$ is a support for $v$. 

However, this approach implies invoking $O(n)$ times a LP solver, which might be too expensive to be practical, since, in addition,
a propagation-based algorithm might call a domain filtering algorithm many times during its execution. To overcome this difficulty we shall 
require further conditions on the matrix $M$. To this end, we define the hypergraph
associated to instance $I$ to be the hypergraph $H$ with $V(H)=V$ and $E(H)=\{S_j \mid 1\leq j\leq m\}$.

Now, assume that $H$ is a network hypergraph
defined by a tree $T$. In this case, one can use specific and more efficient methods like the network simplex algorithm (see for example \cite{Ahuja93}) instead of a general purpose LP solver. However,
it is still possible to do better (and avoid the $O(n)$ calls to the network simplex algorithm) by transforming it into a maximum flow problem. This idea has been used in 
\cite{Maher08} to obtain a domain filtering algorithm for the \sequence{} constraint. 
More precisely, \cite{Maher08} deals with the particular case of network matrices defined by a directed path.
Our approach draws upon \cite{Maher08} and generalizes it to network matrices defined by arbitrary trees. This is done as follows.

Let $P$ be the incidence matrix of $T$. That is, let $t_1,\dots,t_{m+1}$ be an arbitrary ordering of nodes in $T$ and define $P$ to be the 
$((m+1)\times m)$-matrix where $P_{i,j}$ is $+1$ if edge $e_{S_j}$ starts at $t_i$, $-1$ if
$e_{S_j}$ ends at $t_i$, and $0$ otherwise.

Let $r$ be a $m$-ary (column) vector and let $p$ be a path in $T$. We say that $r$ is the {\em indicator vector} of $p$
if for every $j=1,\dots,m$, $r_j$ is the polarity of $e_{S_j}$ wrt. $p$. The next observation follows directly from the definitions.

\begin{observation}
\label{obs:1}
Let $p$ be a path in $T$ and let $r$ be its indicator vector. Then, the $i$th entry, $(Pr)_i$, of $Pr$, is
$+1$ if $t_i$ is the first node in $p$, $-1$ if $t_i$ is the last node in $p$, and $0$ otherwise.
\end{observation}

The next two lemmas follow directly from the previous observation.

\begin{lemma}
$P$ has full rank.
\end{lemma}
\begin{proof}
Let $P'$ be the $(m\times m)$ matrix obtained by removing the last row (corresponding to vertex $t_{m+1}$) and 
consider the $(m\times m)$-matrix $Q$ such that for every $i=1,\dots,m$, the $i$th column of $Q$, which we shall denote as $Q_{*,i}$, is the indicator vector of the  unique path in $T$ starting at $t_i$ and entding at $t_{m+1}$. It follows from Observation \ref{obs:1} that $P'Q$ is the identity matrix. 
\end{proof}

Then, since $P$ has full rank we can obtain an equivalent system $PMz=Pa$ by multiplying both sides of
$Mz=a$ by $P$. Let $N=PM$ and $b=Pa$ (see Example \ref{ex:P}).


\begin{lemma}
In every column of $N$ one entry is $+1$, one entry is $-1$, and all the other entries are $0$.
\end{lemma} 
\begin{proof}
It is only necessary to show that every column, $M_{*,k}$, $k=1,\dots,m+n$ of $M$ is the indicator vector of some directed path in $T$. 
If the variable corresponding to the $k$-column is $x_i$ for some $1\leq i\leq n$ then by construction $M_{*,k}$
is the indicator column of the path associated to $v_i$ in $T$. Otherwise, if the variable corresponding to column $k$ is
$y_j$ for some $1\leq j\leq m$ then $M_{*,k}$ is the indicator vector of the directed path 
containing only edge $e_{S_j}$.
\end{proof}
Hence, matrix $N$ is the incidence matrix of a directed graph $G$.
Note that, by definition, $G$
contains an edge $e_k, k=1,\dots,m+n$ for each variable $z_k$ and a node $u_j, j=1,\dots,m+1$ for each row in $N$ (that is, for every
node in $T$). Define
the capacity of every edge $e_k$ to be $c_k$. Then, feasible solutions of the system correspond precisely to
flows where every node $u_j$ has a supply/demand specified by $b_j$ (more precisely, node $u_j$ has a demand of $b_j$ units if $b_j>0$ and a
suply of $-b_j$ units if $b_j<0$). It is well know that this
problem can be reduced to the (standard) maximum flow problem by adding new source and sink nodes $s, t$ and edges
from $s$ to $u_j$ with capacity $b_j$ whenever $b_j>0$ and from $u_j$ to $t$ with capacity $-b_j$ whenever
$b_j<0$.

\begin{example}
\label{ex:P}\label{ex:I} 
Let $I$ be a boolean instance with variables $\{v_1,\dots,v_6\}$ and constraints:
$C_1=\among(\{v_1,v_4\},\{1\},0,1)$, 
$C_2=\among(\{v_2,v_5\},\{1\},0,1)$, \\
$C_3=\among(\{v_3,v_6\},\{1\},0,1)$, 
$C_4=\among(\{v_1,v_3,v_5\},\{1\},1,1)$, \\
$C_5=\among(\{v_2,v_4,v_6\},\{1\},1,1)$. \\

The specific values for $M$, $a$ and $c$ of the ILP formulation of $I$ are:
{\scriptsize
$$\begin{blockarray}{ccccccccccccccc}
& & & & y_1 & y_2 & y_3  & x_1 & x_2 & x_3 & x_4 & x_5 & x_6\\
\begin{block}{cccc(ccccccccc)c(c)}
   & & C_1 & &  1 & 0 & 0  & 1 & 0 & 0 & 1 & 0 & 0 & & 1 \\
   & & C_2 & &  0 & 1 & 0  & 0 & 1 & 0 & 0 & 1 & 0 & & 1 \\  
M= & & C_3 & &  0 & 0 & 1  & 0 & 0 & 1 & 0 & 0 & 1 & \quad \quad \quad a= \phantom{a} & 1 \\
   & & C_4 & &  0 & 0 & 0  & 1 & 0 & 1 & 0 & 1 & 0 & & 1 \\
   & & C_5 & &  0 & 0 & 0  & 0 & 1 & 0 & 1 & 0 & 1 & & 1 \\
\end{block}
\\
\begin{block}{cccc(ccccccccc)cc}
c^T= & & & &  1 & 1 & 1  & 1 & 1 & 1 & 1 & 1 & 1 &  & \\
\end{block}
\end{blockarray}
$$
}

Note that since $\min_i=\max_i$ for $i=4,5$ we did not need
to add the slack variables $y_4$ and $y_5$.
The hypergraph of this instance 
 is precisely the hypergraph $H$ in Example \ref{ex:hypergraph}. In particular, $h_i$ is the hypererdge corresponding to
the scope of constraint $C_i$ for $i=1,\dots,5$. The matrix $P$ obtained from the tree $T$ defining $H$ is:

{\scriptsize
$$\begin{blockarray}{ccccccc}
& & h_1 & h_2 & h_3 & h_4 & h_5 \\
\begin{block}{cc(ccccc)}
t_0 & \phantom{a} & -1  &  0  &  0  &  0  &  0  \\
t_1 & & 0  &  -1  &  0   &  0 &  0   \\
t_2 & & 0  &  0  &  -1   &  0  &  0  \\  
s_0 & & 0  &  0  &  0   &  1  &  0  \\
s_1 & & 0  &  0  &  0   &  0  &  1   \\
r   & & 1 & 1  &  1   &  -1  &  -1    \\
\end{block}
\end{blockarray}
$$
}

Multiplying $M$ and $c$ by $P$ we obtain:
{\scriptsize
$$
\begin{blockarray}{ccccccccccccccc}
& & & & y_1 & y_2 & y_3  & x_1 & x_2 & x_3 & x_4 & x_5 & x_6\\
\begin{block}{cccc(ccccccccc)c(c)}
   & & t_0 & &  -1  & 0  & 0   & -1 & 0  & 0  & -1 & 0  & 0 &  & -1 \\
   & & t_1 & &  0  & -1  & 0   & 0  & -1 & 0  & 0  & -1 & 0 & & -1 \\
N=PM = & & t_2 & &  0  & 0  & -1   & 0  & 0  & -1 & 0  & 0  & -1 & \quad \quad \quad b=Pa= \phantom{a} & -1 \\
   & & s_0 & & 0  & 0   & 0  & 1  & 0  & 1  & 0  & 1  & 0  & & 1 \\
   & & s_1 & &  0  & 0 & 0  & 0  & 1  & 0  & 1  & 0  & 1  & & 1 \\  
   & & r & &  1  & 1  & 1   & 0  & 0  & 0 & 0  & 0  & 0 & & 1 \\
\end{block}
\end{blockarray}
$$}
The feasible solutions of the previous LP correspond to the feasible flows of
the network in Figure \ref{fig:2a}, where nodes $s_0,s_1$ and $r$ have a supply of one
unit of flow and nodes $y_0,y_1,y_2$ have a demand of one unit of flow. Figure
\ref{fig:2b} contains the result of transforming the network in Figure \ref{fig:2a}
to a (standard) max flow problem. In both networks all edges have capacity $1$.

\captionsetup[subfigure]{labelformat=empty}

\begin{figure}
\centering
	\begin{subfigure}[t]{1.4in}
		\centering
\tikzstyle{vertex}=[circle,fill=black!25,minimum size=12pt,inner sep=0pt]
\tikzstyle{edge} = [draw,thick,-]
\tikzstyle{diedge} = [draw,thick,->]
\tikzstyle{weight} = [font=\small]

\begin{tikzpicture}[scale=0.5,auto]

\foreach \pos/\name in {{(-2.9,-1.5)/{-1}},{(-2.9,1.5)/{-1}},{(-2.9,0)/{-1}},{(2.7,1.5)/1},{(2.7,0)/1},{(2.7,-1.5)/1}}
      \node at \pos {$\name$};
\foreach \pos/\name in {{(-2,-1.5)/r},{(-2,1.5)/{s_0}},{(-2,0)/{s_1}},{(2,1.5)/{t_0}},{(2,0)/{t_1}},{(2,-1.5)/{t_2}}}
        \node[vertex] (\name) at \pos {$\name$};
\foreach \source/ \dest  in {{s_0}/{t_0},{s_0}/{t_1},{s_0}/{t_2},{s_1}/{t_0},{s_1}/{t_1},{s_1}/{t_2},r/{t_0},r/{t_1},r/{t_2}}
      \path[diedge] (\source) -- node {} (\dest);

\end{tikzpicture}

		\caption{\centering Fig.\,{\ref{fig:2a}}:\,\,{Network with suppies/demands}}\label{fig:2a}		
	\end{subfigure}
	\quad \quad \quad \quad \quad \quad \quad
	\begin{subfigure}[t]{1.6in}
		\centering
\tikzstyle{vertex}=[circle,fill=black!25,minimum size=12pt,inner sep=0pt]
\tikzstyle{edge} = [draw,thick,-]
\tikzstyle{diedge} = [draw,thick,->]
\tikzstyle{weight} = [font=\small]
\begin{tikzpicture}[scale=0.5,auto]
\foreach \pos/\name in {{(-2,-1.5)/r},{(-2,1.5)/{s_0}},{(-2,0)/{s_1}},{(2,1.5)/{t_0}},{(2,0)/{t_1}},{(2,-1.5)/{t_2}},{(-4,0)/s},{(4,0)/t}}
        \node[vertex] (\name) at \pos {$\name$};
\foreach \source/ \dest  in {{s_0}/{t_0},{s_0}/{t_1},{s_0}/{t_2},{s_1}/{t_0},{s_1}/{t_1},{s_1}/{t_2},r/{t_0},r/{t_1},r/{t_2},s/{s_0},s/{s_1},s/r,{t_0}/t,{t_1}/t,{t_2}/t}
      \path[diedge] (\source) -- node {} (\dest);

\end{tikzpicture}
		\caption{\centering Fig.\,{\ref{fig:2b}}:\,\,{Standard flow network obtained from (2a)}}\label{fig:2b}
	\end{subfigure}
\end{figure}

\end{example}

It follows from this construction that for every $1\leq i\leq n$ and every $d=\{0,1\}$, there is a solution $s$
of $I$ with $s(v_i)=d$ if and only if there is a {\em saturating} flow (that is, a flow where all the edges leaving $s$ or entering $t$ are at full capacity)
such that the edge associated to $x_i$ carries $d$ units of flow. R\'egin \cite{Regin94,Regin96} has shown that this later condition can be tested 
{\em simultaneously} for
all $1\leq i\leq n$ and $d\in\{0,1\}$ by finding a maximal flow and computing the strongly connected components of its residual graph. 
Finding a maximal flow of a network with integral capacities can be done in time $O(\min(v^{2/3},e^{1/2})e\log(v^2/e)\log u)$ 
using Goldberg and Rao's algorithm \cite{GoldbergR98}
where $v$ is the number of vertices, $e$ is the number of edges, and $u$ is the maximum capacity of an edge. Computing the strongly connected components of the residual graph takes 
$O(v+e)$ time using Tarjan's algorithm \cite{Tarjan72}. By construction, the network derived by our algorithm satisfies $e\leq n+2m$ and $v\leq m+1$.
Furthermore, it is not difficult to
see that $u\leq mn$. Indeed, note that the capacity of any edge is either some entry, $c_i$, of vector $c$ or the absolute value of some entry, $b_i$, of vector $b$. 
It follows directly from the definition of $c$, that all its entries are at most $n$. As for $b$, the claim follows from the
fact that $b=Pa$ where, by construction, $a$ has $m$ entries where every entry is in the range $\{1,\dots,n\}$ and every entry in $P$ is in $\{-1,0,1\}$.


Hence, if we define $f(n,m)$ to be $\operatorname{min}(m^{2/3},(n+m)^{1/2})(n+m)\log(m^2/(n+m))\log mn)$ we have:
\begin{lemma}
\label{le:boolean}
There is a domain filtering algorithm for conjunctions of boolean among constraints whose associated hypergraph is a network hypergraph, which runs in time 
$O(f(n,m))$ where $n$ is the number of variables and $m$ is the number of constraints, assuming the instance is presented as a network flow problem.
\end{lemma}

It is customary, when analizying the time complexity of domain filtering to report, additionally, the so-called time complexity 'down a branch of a search tree' which consists in the aggregate time complexity of successive calls to the algorithm, when at each new call, the list of some of the variables has
been decreased (as in the execution of a propagation-search based solver). It was observed again by R\'egin \cite{Regin96} that, in this setting, it is not necessary, to solve the flow problem from scratch at each call, leading to a considerable redution in total time. Applying the scheme in \cite{Regin96}, we obtain
that the time complexity down a branch of a search tree of our algorithm is $O(n(n+m))$. We ommit the details because they are faily standard (see \cite{Regin96}).

There are minor variants (leading to the same asymptotic complexity) obtained by modifying the treatment of the slack variables. Here we will discuss
two of them. In the first variant, used in \cite{Maher08}, one encodes a constraint $C_j=(S_j,\{1\},\min_j,\max_j)$ with two equations $y_j+\sum_{v_i\in S_j} x_i=\max_j$, and $-z_j+\sum_{v_i\in S_j} x_i=\min_j$ where $y_j$ and $z_j$ are new slack variables satisfying $0\leq y_j,z_j$. This encoding produces a network
that has $m$ more nodes and edges. In a second variant, one encodes a constraint $C_j$ with the equation $-y_j+\sum_{v_i\in S_j} x_i=0$ where $y_j$ satisfies
$\min_j\leq y_j\leq\max_j$. Under this encoding, our approach produces a network problem where, instead of having nodes with specified suply or demand, 
we have edges with minimum demand. The asymptotic time bounds in all variants are identical.

If the instance is not presented as a network flow problem then one would need to add the cost of transforming the instance into it. However
this cost would be easily amortized as a domain filtering algorithm is invoked several times during the execution of a constraint solver\footnote{In fact, as shown in \cite{Regin94,Regin96}, it is only necessary to solve the max flow problem
during the first invocation so it could be argued that a more realistic bound on the running time of the algorithm is $O(n+m)$.}. 
Furthermore, in practical scenarios, the conjunction of among constraints will 
encode a global constraint from a catalog of available global constraints.  Hence, it is reasonable to assume that the formulation of the global constraint as a network flow problem can be precomputed.

This approach can be generalized to non-boolean domains via boolean encoding. The choice of boolean encoding might depend on the particular
instance at hand but for concreteness we will fix 
one. The {\em canonical booleanization} (see Example \ref{ex:toy}) of a conjunction $I=(V,D,L,\mathcal C)$
of among constraints with $|D|\geq 3$ is the boolean instance $(V\times D,$\{0,1\}$,L_b,{\mathcal C}_b)$ where 
$L_b(v,d)=\{0,1\}$ if $d\in L(v)$ and $\{0\}$ otherwise, and ${\mathcal C}_b$ contains:
\begin{itemize}
\item $\among(S\times R,\{1\},\min,\max)$ for every constraint $\among(S,R,\min,\max)\in{\mathcal C}$, and
\item $\among(\{v\}\times D,\{1\},1,1)$ for every variable $v\in V$. This family of constraints
are called {\em non-empty assignment constraints}.
\end{itemize}
That is, the intended meaning of the encoding is that $(v,d)\in V\times D$ is true whenever $v$ takes value $d$.

We define the hypergraph $H$ associated to $I$ to be the hypergraph associated
to the canonical booleanization of $I$. That is, $V(H)$ is $V\times D$ and $E(H)$ contains
hyperedge $\{(v,d) \mid v\in S,d\in R\}=S\times R$ for every constraint $\among(S,R,\min,\max)$ in ${\mathcal C}$, and
hyperedge $\{(v,d) \mid d\in D\}=\{v\}\times D$ for every variable $v\in V$.
Thus, for arbitrary domains, we have:

\begin{corollary}
\label{co:arbitrary}
There is a domain filtering algorithm for conjunctions $(V,D,L,{\mathcal C})$ of among constraints whose associated hypergraph
is a network hypergraph, which runs in time $O(f(n,m))$ where $n=\sum_{v\in V} |L(v)|$  and $m=|{\mathcal C}|+|V|$, assuming the instance is presented as a network flow problem.
\end{corollary}
\begin{proof}
It just follows from observing that the canonical booleanization of instance $(V,D,L,{\mathcal C})$ has $n=\sum_{v\in V} |L(v)|$ variables and $m=|{\mathcal C}|+|V|$ constraints.
\end{proof}



\section{Some applications}
\label{sec:applications}

The aim of this section is to provide evidence that the kind of \cac{}s covered by the approach developped in Section \ref{sec:alg} are often encountered in practice. To this end, we shall revisit several families of \cac{}s for which domain filtering have been previously introduced and show how they can solved and, in some cases generalized, using our algorithm. Furthermore, we shall compare, whenever possible, the time complexity bounds of our algorithm with other state-of-the-art algorithms for the same problem. Like other flow-based algorithms, the algorithm proposed here has very good time complexity down a branch of the search tree. However, we will only consider
in our comparison the cost of calling the algorithm just once. This is due to the fact that, once the size of the network produced is under a certain threshold, 
the total cost down a branch of the search tree is dominated by the cost of the incremental updates and, hence, it cannot be used to assess the comparative
quality of different flow-based algorithms. Furthermore, we will try to compare, whenever possible, the parameteres of the obtained network flow problem 
(number of nodes, edges, capacities of the edges) instead of the actual running time since the latter is dependend on the choice of the max-flow algorithm.
Somewhat surprisingly, in many of the cases, even if we did not attempt any fine-tuning,  the network produced by the algorithm is essentially equivalent to the network produced by specific algorithms. 

\subsection{Disjoint constraints}
As a warm up we shall consider the \globalcardinality{} and \alldifferent{} constraints. We have
seen in Example \ref{ex:gcc} that both can be
formulated as a conjunction $(V,D,L,{\mathcal C})$ of among constraints of the form 
$C_d=\among(V,\{d\},\min_d,\max_d), d\in D$. Note that both have
the same associated hypergraph $H$ with node-set, $V\times D$,
and edge set $E(H)$ containing hyperedge $h_v=\{v\}\times D$ for every $v\in V$, and
hyperedge $h_d=V\times\{d\}$, for every $d\in D$.

It is not difficult to 
see (see Example \ref{ex:toy}) that $H$ is defined by the tree $T$ defined as follows:
The node-set of $T$ consists of $\{r\}\cup V\cup D$ where $r$ is a new node. The edge-set of $T$ contains an edge from every $v\in V$ to $r$ (associated 
to $h_v$) and from $r$ to every node $d\in D$ (associated to $h_d$). 
Consequently, both \globalcardinality{}
and \alldifferent{} are solvable by our algorithm. The network
for the \globalcardinality{} using the abovementioned tree $T$ has
$e=O(|V||D|)$ edges, $v=|V|+|D|+3$ nodes, and the 
maximum capacity, $u$, of an edge is at most $|D||V|$.
Hence, it follows that the total running time of our algorithm for the \globalcardinality{} constraint is $O(\min(v^{2/3},e^{1/2})e\log(v^2/e)\log u)$. 
R\'egin's algorithm \cite{Regin96} has a $O(|V|^2|D|)$ complexity which is better when $|V|\in O(|D|^{2/3})$ but the comparison between the two bounds is not very 
meaningfull because it mainly reflects a different choice max flow algorithm. Indeed, the network produced by both
algorithms are very similar. In particular, the network obtained using the second variant discussed after Lemma \ref{le:boolean} is essentially the same described
in \cite{Regin96}. The only difference is that the network obtained by our algorithm contains one extra node. 
In the particular case of the \alldifferent{} constraint, \cite{Regin94} shows
how to produce a bipartite matching problem that can be solved using specialized algorithms, such as \cite{HopcroftK73}, leading to a total time
complexity of $O(|V|^{5/2})$ which is better than ours.

\begin{example}
\label{ex:toy}
Consider the constraint $\alldifferent(s_0,s_1)$ where the list of each
variable contains the following three values: $t_0,t_1,t_2$. Then, $\alldifferent(s_0,s_1)$ is encoded as a \cac{} $I$ with the following constraints:
$\among(\{s_0,s_1\},\{t_0\},0,1)$, $\among(\{s_0,s_1\},\{t_1\},0,1)$, $\among(\{s_0,s_1\},\{t_2\},0,1)$.

The canonical booleanization of $I$ has variables $\{s_0,s_1\}\times\{t_0,t_1,t_2\}$ and
constraints $C_1=\among(\{s_0,s_1\}\times\{t_0\},\{1\},0,1)$, $C_2=\among(\{s_0,s_1\}\times\{t_1\},\{1\},0,1)$,
$C_3=\among(\{s_0,s_1\}\times\{t_2\},\{1\},0,1),C_4=\among(\{s_0\}\times\{t_0,t_1,t_2\},\{1\},1,1)$, 
and $C_5=\among(\{s_1\}\times\{t_0,t_1,t_2\},\{1\},1,1)$. Observe that this instance is, under the renaming 
$v_i\mapsto (s_{i-1 \mod 2},t_{i-1 \mod 3})$, the same instance than we have considered previously in
Example \ref{ex:I}. The network flow problem that our algorithm derives for this instance (see Example \ref{ex:P})
is almost identical to the one derived in \cite{Regin96}. Indeed, the network obtained in \cite{Regin96} does not have
node $r$ and, instead, requires that the demand of nodes $t_0,t_1,t_2$ is at most one (instead of {\em exactly} one).
\end{example}

A simple analysis reveals that the same approach can be generalized to instances $(V,D,L,{\mathcal C})$ satisfying
the following {\em disjointedness} condition: for every pair of constraints $\among(S,R,\min,\max)$ and $\among(S',R',\min',\max')$
 in ${\mathcal C}$, $(S\times R)\cap(S'\times R')=\emptyset$. The tractability of such instances was, to the best 
of our knowledge, not known before. The particular case in which $R\cap R'=\emptyset$ has been previously
shown in \cite{Regin05} using a different approach.  The proof given in \cite{Regin05} does not construct a flow problem 
nor gives run-time bounds so we ommit a comparison.

\subsection{Domains consisting of subsets}
Consider the following generalization of our setting where in a \cac{} $(V,D,L,{\mathcal C})$
every variable $v$ must be assigned to a subset of $L(v)$ (instead of a single element). 
In this case, the semantics of the among constraint need to be generalized as well. Instead, we will say 
a constraint $\among(S,R,\min,\max)$ is satisfied by a mapping $s:V\rightarrow 2^D$ if
$\min\leq\sum_{v\in S} |s(v)\cap R|\leq\max$. To avoid confusion we shall refer to this variant
of the among constraint as {\em set among} constraint.

For example, the \symgcc{} constraint \cite{Kocjan04} is precisely a conjunction $(V,D,L,{\mathcal C})$ of set among constraints
of the form $\among(V,\{d\},\min,\max)$ where $d$ is a singleton which, additionally, might contain
constraints of the form $\among(\{v\},D,\min,\max)$ restricting the size of the image of a variable $v$.

It is fairly easy to reduce a conjunction of set among constraints $I=(V,D,L,\mathcal C)$ 
to a conjunction of (ordinary) among constraints over a boolean domain. Indeed, one only needs to 
construct the instance  $(V\times D,$\{0,1\}$,L_b,{\mathcal C}_b)$ where 
$L_b(v,d)=\{0,1\}$ if $d\in L(v)$ and $\{0\}$ otherwise, and ${\mathcal C}_b$ contains 
$\among(S\times R,\{1\},\min,\max)$ for every constraint $\among(S,R,\min,\max)\in{\mathcal C}$.
Note that the instance thus constructed corresponds exactly to the result of removing the non-empty assignment constraints to the
canonical booleanization of $I$ (now regarded as a conjunction of ordinary among constraints).
It is then easy to observe that if $(V,D,L,\mathcal C)$ encodes a \symgcc{} constraint then the
resulting boolean instance has the same hypergraph, $H$, than the $\globalcardinality$ constraint and hence, $H$
is a network matrix. The algorithm in \cite{Kocjan04} follows closely that in \cite{Regin96} for
the \globalcardinality{} constraints, and, in particular, has the same time bounds. Consequently
the network flow derived by our algorithm is obtained, again, by adding one extra node with small capacities in the edges to the network 
introduced in \cite{Kocjan04}. 

\subsection{The sequence constraint}
\label{sec:sequence}
The \sequence{} constraint \cite{Beldiceanu94} corresponds to
instances $(\{v_1,\dots,v_n\},D,L,{\mathcal C})$ with constraints
$\among(\{v_i,\dots,v_{i+k}\},R,\min,\max), i=1,\dots,n-k$ for some 
fixed integers $\min,\max,k$, and fixed $R\subseteq D$. It is not difficult to see that 
the hypergraph of the canonical booleanization of the \sequence{} constraint is {\em not} a network hypergraph. However,
as shown in \cite{Maher08} one obtains an equivalent instance with a network hypergraph using a different encoding
in which for every original variable $v_i\in V$, we have a boolean variable $x_i$ which is intended to be true
whenever $v_i$ takes a value in $R$ and false otherwise. Under this alternative encoding we obtain a boolean instance $I$
which consists of constraints $\among(\{x_i,\dots,x_{i+k}\},\{1\},\min,\max), i=1,\dots,n-k$. It is shown in \cite{Maher08} that the hypergraph $H$ of the
boolean instance $I$ obtained under this encoding satisfies the so-called consecutive-ones property which implies that $H$ is
defined by a tree $T$ consisting of a single directed path. Indeed, the network flow obtained by our approach
is identical to the one derived in \cite{Maher08} if one encodes \among{} constraints using the first variant
discussed after Lemma \ref{le:boolean}. Applying Lemma \ref{le:boolean} and noting that, in the particular 
case of the \sequence{} constraint, we have $m=O(n)$ we obtain the bound $O(n^{3/2}\log^2 n)$. By inspecting
closely the proof of Lemma \ref{le:boolean} this bound can be slightly improved 
(see 
Appendix C) 
to $O(n^{3/2}\log n\log \max)$
coinciding with the bound given in \cite{Maher08}, which is not surprising since both networks are
essentially equivalent. To the best of our knowledge $O(n^{3/2}\log n\log \max)$ is the best bound among
all complete domain consisteny algorithms for the problem, jointly with the algorithm proposed in \cite{HoevePRS06} which, 
with time complexity $O(n2^k)$, offers gives better bounds when $k\ll n$.


 

\subsection{\TFO{} model}
\label{sec:tfo}
The \TFO{} model was introduced by Razgon et al. \cite{RazgonOP07} as a generalization of several common global constraints. Formally, 
a \TFO{} model is a triple $(V,F_1,F_2)$ where $V$ is a finite set of vertices and $F_1$ and $F_2$ are nonempty families of subsets of $V$
such that two sets that belong to the same family are either disjoint or contained in each other. Each set $Y$ in $F_1\cup F_2$
is associated with two non-negative integers $\min_Y,\max_Y\leq |Y|$.
A subset $X$ of $V$
is said to be {\em valid} if $\min_Y\leq X\cap Y\leq \max_Y$ for every $Y\in F_1\cup F_2$. The task is to find the largest valid subset.
Although the methods introduced in the present paper can be generalized to deal as well with optimization version we will consider
only now the feasibility problem consisting in finding a valid subset (or report that none exists). 

First, note that the existence of a valid subset in a \TFO{} model can be formulated naturally as a satisfiability problem for a 
combination of among constraints. Indeed, there is a one-to-one correspondence between the valid subsets of $(V,F_1,F_2)$ and the 
solutions of the instance $(V,\{0,1\},L,{\mathcal C})$ where  ${\mathcal C}$
contains the constraints $\among(Y,\{1\},\min_Y,\max_Y), Y\in F_1\cup F_2$, and the list $L(v)$ of every variable $v\in V$ is $\{0,1\}$. The hypergraph $H$ associated to this instance is $(V,F_1\cup F_2)$. 
It follows
directly from Lemmas \ref{le:rootedtree} and \ref{le:glue} (see Appendix)
that $H$ is a network hypergraph and hence one can use
our approach to decide the existence of a feasible solution of a \TFO{} model. It turns out that the network introduced in \cite{RazgonOP07}
is essentially equivalent to the network flow problem that would be obtained by our approach using the second variant 
described after Lemma \ref{le:boolean}. It is not meaningfull to compare the running time of our algorithm
with that of \cite{RazgonOP07} since it deals with an optimization variant.

\subsection{Conjunction of among constraints with full domain}
\label{sec:fulldomain}

Some global constraints studied in the literature correspond to conjunctions
$(V,D,L,{\mathcal C})$ of among constraints where the scope of every constraint is the full set $V$
of variables.
This class contains, of course, the \globalcardinality{} constraint
and also several others, since we do not require $R$ to be a singleton. 
For example, the \ordereddistribute{} constraint introduced by Petit and R\'egin \cite{Petit11} 
can be encoded as conjunction $(V,D,L,{\mathcal C})$ of among constraint where where the domain $D$ has some arbitrary
(but fixed) ordering $d_1,\dots,d_{|D|}$ and in every constraint $\among(S,R,\min,\max)$,
$S=V$ and $R$ is of the form $\{d_i,\dots,d_{|D|}\}$.

We shall show that the hypergraph of the conjunction of among constraints defining \ordereddistribute{}
is a network hypergraph. Indeed, with some extra work we have managed to completely characterize 
all \cac{} instances containing only constraints with full scope that have an associated network hypergraph.



\begin{theorem}
\label{the:fulldomain}
Let $I=(V,D,L,{\mathcal C})$ be a conjunction of among constraints with $|D|\geq 3$ such that the scope of
each constraint is $V$. Then, the following are equivalent:
\begin{enumerate}
\item The hypergraph of the canonical booleanization of $I$ is a network hypergraph.
\item For every pair of constraints in $\mathcal C$, their ranges are disjoint or one of them 
contained in the other.
\end{enumerate}
\end{theorem}

In the particular case of \ordereddistribute{} constraint, the network obtained by our approach is very related
to the network introduced in section (\cite{Petit11}, Section V.A). More precisely, the abovementioned network 
is essentially equivalent to the network that would be obtained by our approach using the second variant
described after Lemma \ref{le:boolean}. However, our algorithm is far from optimal. In particular,
a complete filtering algorithm with time complexity $O(|V|+|D|)$ is given also in \cite{Petit11}.

\subsection{Adding new among constraints to a  \globalcardinality{} constraint}

Let $(V,D,L,{\mathcal C})$ be a conjunction of among constraints encoding the $\globalcardinality$ constraint
(see Example \ref{ex:gcc}) and assume that we are interested in adding several new among constraints to
it. In general, we might end up with a hard instance but depending on the shape of the new constraints 
we might perhaps still preserve
tractability. Which among constraint we might safely add? This question has been addressed by R\'egin \cite{Regin05}.
In particular, \cite{Regin05} shows that the domain filtering problem is still tractable whenever:
\begin{itemize}
\item[(a)] every new constraint added has scope $V$ and, furthermore, the ranges of every pair of new constraints are disjoint, or
\item[(b)] every new constraint added has range $D$ and furthermore, the scopes of every pair of new constraints are disjoint.
\end{itemize}

We can explore this question by inquiring which families of constraints can be added to an instance $(V,D,L,{\mathcal C})$
encoding \globalcardinality{} such that its associate hypergraph is still a network hypergraph. Somewhat surprisingly we
can solve completely this question (see Theorem \ref{the:addtogcc}). This is due to the fact that the
presence of the global cardinality constraint restricts very much the shape of the tree defining
the hypergraph of the instance. 

\begin{theorem}
\label{the:addtogcc}
Let $I=(V,D,L,{\mathcal C})$ be a conjunction of among constraints containing a global cardinality constraint with scope $V$
with $|D|\geq 3$.
Then the following are equivalent:
\begin{enumerate}
\item The hypergraph of the canonical booleanization of $I$ is a network hypergraph.
\item In every constraint in ${\mathcal C}$, the scope is a singleton or $V$, or the range is a singleton or $D$.
Furthermore, for every pair $\among(S_1,R_1,\min_1,\max_1)$, $\among(S_2,R_2,\min_2,\max_2)$ of constraints 
in ${\mathcal C}$ the following two conditions hold:
\begin{enumerate}
\item If $S_1=S_2=V$ or $S_1=S_2=\{v\}$ for some $v\in V$ then $R_1$ and $R_2$ are disjoint or one of them is contained in the other. 
\item If $R_1=R_2=D$ or $R_1=R_2=\{d\}$ for some $d\in D$ then $S_1$ and $S_2$ are disjoint or one of them is contained in the other.
\end{enumerate}
\end{enumerate}
\end{theorem}

Note that the previous theorem covers cases (a) and (b) from \cite{Regin05} described at the beginning of this section. The network
produced in \cite{Regin05} for the case (a) is essentially equivalent to the one derived by our approach using the second variant described
after Lemma \ref{le:boolean}. For the case (b) \cite{Regin05} does not construct a flow problem nor gives run-time bounds so
we ommit a comparison.

\section*{Acknowledgments} 
The author would like to thank the anonimous referees for many useful comments. This work was supported 
by the MEIC under grant TIN2016-76573-C2-1-P and the MECD under grant PRX16/00266.

\bibliographystyle{plain}
\bibliography{biblio}

\section*{Appendix A: Tractable subcases of the domain filtering problem restricting only the scope or the range in the constraints}

If ${\mathcal I}$ is a set of conjunctions of among constraints, we shall
denote by $\filter({\mathcal I})$ the restriction of the domain filtering problem to instances in ${\mathcal I}$. Our ultimate goal
would be to characterize precisely for which sets ${\mathcal I}$, $\filter({\mathcal I})$ has efficient algorithms.

The first question that we address is the following: which subcases of the problem can be explained by considering {\em only} the scopes of the
constraints? The scopes occurring in an instance can be characterized by an hypergraph. More precisely, let $I=(V,D,L,{\mathcal C})$ be a conjunction of among constraints. The {\em scope hypergraph} of $I$ is the hypergraph $H$ with node set $V(H)=V$
and that contains an hyperedge for every scope occurring in a constraint in ${\mathcal C}$. Formally, 
$$E(H)=\{ S \mid \among(S,R,\min,\max)\in {\mathcal C}\}$$

Let ${\mathcal H}$ be a (possibly infinite) set of hypergraphs. We denote by $\scope({\mathcal H)}$ the collection of instances $I$ of the
domain filtering problem whose scope hypergraph belongs to ${\mathcal H}$. Our question can be formalized in the following manner: for which sets, ${\mathcal H}$, of hypergraphs, is $\filter(\scope(\mathcal H))$ efficiently solvable? Note that if $I$ is an instance whose scope hypergraph $H$ is a subhypergraph of 
some $H'\in {\mathcal H}$ then we can construct, by adding superfluous new constraints to $I$, an equivalent new instance whose scope hypergraph is $H'$. Hence, we can 
assume that ${\mathcal H}$ is closed under taking subhypergraphs.

We shall solve completely this question assuming some mild technical assumptions. In order to state our result we need a few definitions from graph theory and
parameterized complexity. The Gaifman graph of
a hypergraph $H$, denoted $\gaifman(H)$ is the graph where the node-set is $V(H)$ and the edge-set contains all pairs $\{u,v\}$ such that
there is an hyperedge $h$ in $H$ with $\{u,v\}\subseteq h$. A {\em tree-decomposition} of a graph $G$ is a pair $(T,\beta)$ where
$T$ is an (ordinary, not oriented) tree and $\beta:V(T)\rightarrow 2^{V(G)}$ is a mapping such that the following conditions are satisfied:
\begin{enumerate}
\item For every node $v\in V(G)$, the set $\{x\in V(T) \mid v\in \beta(x) \}$ is non-empty and connected in $T$.
\item For every edge $\{u,v\}\in E(G)$, there is a node $x\in V(T)$ such that $\{u,v\}\in \beta(x)$.
\end{enumerate}
The {\em width} of a tree-decomposition $(T,\beta)$ is $\max\{|\beta(x)|-1 \mid x\in V(T)\}$ and the tree-width of $G$ is defined to be 
the minimum $w$ such that $G$ has a tree-decomposition of width $w$.

We will also need some notions and basic facts from parameterized complexity theory. A {\em parameterized problem} over some alphabet $\Sigma$
is a pair $(P,\kappa)$ consisting of a problem $P\subseteq\Sigma^*$ and a polynomial time mapping $\kappa:\Sigma^*\rightarrow\mathbb{N}$, called its
parameter. A parameterized problem $(P,\kappa)$ over $\Sigma$ is {\em fixed-parameter tractable} if there is a
computable function $f:\mathbb{N}\rightarrow\mathbb{N}$ and an algorithm that decides if a given instance $x\in\Sigma^*$ belongs to $P$ in
time $f(\kappa(x))\cdot |x|^{O(1)}$. $\fpt$ denotes the class of fixed-parameter tractable problems. Hence, 
the notion of fixed-parameter tractability relaxes the classical notion of polynomial-time solvability, by admitting running times that
are exponential in the parameter, which is expected to be small.

The analogous of NP in parameterized complexity is the class $\wone$ which is conjectured to contain {\em strictly} $\fpt$. We will omit the
definition of $\wone$ since it is not needed in our proofs and refer the reader to \cite{Flum06}.

\begin{theorem}
\label{the:hard}
Assume $\fpt\neq\wone$. For every recursively class of hypergraphs $\mathcal H$, $\filter(\scope(\mathcal H))$
is polynomial-time solvable if and only if there exists some natural number $w$ such that the tree-width of the Gaifman graph of every hypergraph in ${\mathcal H}$ is at most $w$.
\end{theorem}
\begin{proof}
It is well know \cite{Dechter89:tree,Freuder90:complexity} that the set of all CSP instances whose scope hypergraph has tree-width at most $w$ for some fixed $w$ is solvable in polynomial time. Hence
it only remains to show the 'only if' part. This follows from a result of F\"arnquivst and Jonsson \cite{Farnqvist07}. In order to state it, we need to introduce some background.

In the list homomorphism problem (for graphs), we are given two graphs $G$, $J$ and a mapping $L:V(G)\rightarrow 2^{V(J)}$ called {\em list}. The goal is to decide whether there
exists a mapping
$h:V(G)\rightarrow V(J)$ satisfying the following two conditions:
\begin{enumerate}
\item $h(v)\in L(v)$ for every $v\in V(G)$.
\item $(h(v_1),h(v_2))\in E(J)$ for every $(v_1,v_2)\in E(G)$.
\end{enumerate}
Such a mapping is called a {\em solution} of $(G,J,L)$.  F\"arnquivst and Jonsson have shown that the parameterized version of the problem (parameterized by the
size of $V(G)$) is $\wone$-hard. Indeed, the problem is $\wone$-hard even if the input graph, $G$, is guaranteed to belong to a previously fixed ${\mathcal G}$ of graphs, 
provided ${\mathcal G}$ has unbounded tree-width. Formally, let ${\mathcal G}$ be a set of graphs and define $\plhom({\mathcal G},\_)$ to be the problem:
\begin{itemize}
\item INPUT: graphs $G$, $J$ with $G\in{\mathcal G}$, and a mapping $L:V(G)\rightarrow 2^{V(J)}$. 
\item PARAMETER: $|V(G)|$.
\item GOAL: Decide whether $(G,J,L)$ has a solution.
\end{itemize}

We are finally ready to state the theorem from \cite{Farnqvist07} that we shall use.

\begin{theorem}(Lemma 3 in \cite{Farnqvist07})
Let ${\mathcal C}$ be a recursively class of graphs that does not have bounded tree-width. Then $\plhom({\mathcal C},\_)$ is $\wone$-hard.
\end{theorem}

We note here that, in order to simplify the exposition, we have taken the liberty to adapt the statement in \cite{Farnqvist07}. We are now ready to complete our proof.
Assume, towards a contradiction, that $\mathcal H$ is a set of hypergraphs with unbounded tree-width such that $\filter(\scope(\mathcal H))$ is solvable in polynomial time. 
Let $\mathcal G=\{\gaifman(H) \mid H\in{\mathcal H}\}$. It follows directly by the assumptions 
on ${\mathcal H}$ that ${\mathcal G}$ has unbounded tree-width and is recursively enumerable.  We shall give
an FPT algorithm for $\plhom({\mathcal G},\_)$. 

The algorithm is as follows. Let $(G,J,L)$ be any instance of the list homomorphism problem with $G\in{\mathcal G}$. Enumerate the hypergraphs in ${\mathcal H}$ until 
finding an hypergraph $H\in{\mathcal H}$ whose Gaifman graph is  $G$. 
Construct the instance $I$ in $\scope({\mathcal H})$ where the set of variables is $V(H)(=V(G))$, the domain is $V(H)\times V(J)$,
the list of every node $v\in V(H)$ is $\{v\}\times L(v)$, and the constraints are defined as follows:
\begin{itemize}
\item For every $(v_1,v_2)\in E(G)$, and for every $(a_1,a_2)\not\in E(H)$, include the constraint
$\among(h,R,0,1)$ where $h$ is any hyperedge in $H$ containing $\{v_1,v_2\}$ and
$R=\{(v_1,a_1),(v_2,a_2)\}$.
\end{itemize}
Note that, by construction, the scope hypergraph of $I$ is a subhypergraph of $H$ and, hence, $I$ is an instance of $\scope({\mathcal H})$. Since $\filter(\scope(\mathcal H))$
is, by assumption, solvable in polynomial time then one can also decide the satisfiability of $I$ is polynomial time. Finally, return 'yes' if $I$ is satisfiable and 'no' otherwise.
Note that the time required in finding $H$ depends only on $G$ whereas the time required in constructing and solving $I$ is polynomial on the size of $(G,J,L)$. Hence, the algorithm
just defined is $\fpt$.  It is easy to see that it correctly solves $(G,J,L)$. Indeed, let $s$ be any mapping $s:V(H)\rightarrow V(H)\times V(J)$ which we can write as $s(v)=(s_1(v),s_2(v))$ with $s_1:V(H)\rightarrow V(H)$ and $s_2:V(H)\rightarrow V(J)$. It follows directly
from the construction of $I$ that $s$ is a solution of $I$ if and only if $s_1$ is the identity (formally, $s(v)=v$ for every $v\in V(H)$) and $s_2$ is a solution of instance 
$(G,J,L)$. Hence, $I$ is satisfiable if and only if so is $(G,J,L)$. 

\end{proof}

It is not difficult to show, following \cite{Grohe07}, that the condition $\fpt\neq\wone$ cannot
be weakened (for example by requiring, instead, only P$\neq$NP). Indeed, if $\fpt=\wone$ then there exists
a family $\mathcal H$ of hypergraphs of unbounded tree-width such that $\filter({\mathcal H})$ is polynomial-time solvable.

Finally, note that the hardness part of Theorem \ref{the:hard} holds even for conjunctions of among constraint whose range has cardinality $2$. 
Note, that the hardness direction does not hold any more if one requires that the among constraints have range of cardinality $1$. Indeed, the \alldifferent{}
constraint is encoded by a conjunction of among constrains whose scope hypergraph can have arbitrary large tree-width and the range of every constraint is a singleton.

Secondly, we turn our attention to the range and investigate which restrictions on the range of among constraints guarantee that the domain filtering problem is solvable in polynomial time. To this end we can define the {\em range hypergraph} in a similar way to the scope hypergraph and define, for every
set ${\mathcal H}$ of hypergraphs, $\range(\mathcal H)$ to be the set of all conjunctions of among constraints whose range hypergraph belongs to ${\mathcal H}$.
The next theorem, which is straightforward, shows that all non-trivial hypergraphs give rise to hard problems.

\begin{theorem}
Assume P$\neq$NP. For every hypergraph $H$, $\filter(\range(\mathcal H))$ is polynomial-time solvable if and only every hypergraph $H\in {\mathcal H}$
has only trivial hyperedges (that is, if for every $H\in {\mathcal H}$ and every $h\in E(H)$, $h=\emptyset$ or $h=V(H)$).
\end{theorem}
\begin{proof}
The 'if' direction is trivial. Indeed, if the range hypergraph of an instance $I$ has only trivial hyperedges it follows that every among constraint in it is either
superfluous (in the sense that it does not enforce any restriction) or unsatisfiable. For the 'only if' direction, assume that 
${\mathcal H}$ has an hypergraph $H$ with a non trivial hyperedge $h$. We define a reduction from \oneinthree{}
which is the following NP-complete \cite{Garey79} problem:
\begin{itemize}
\item INPUT: An hypergraph $J$ where all the hyperedges have cardinality $3$.
\item GOAL:  Decide whether there exists some $X\subseteq V(J)$ such that $h\cap X=1$ for every $h\in E(J)$.
\end{itemize}
The reduction is as follows: given an instance $J$, of \oneinthree{}, construct the instance $I=(V,D,L,{\mathcal C})\in\range({\mathcal H})$ 
where $V=V(J)$, $D=V(H)$, $L(v)=V(H)$ for every $v\in V$, and ${\mathcal C}$ contains the constraints $(S,h,1,1)$ for every $S\in E(J)$. It is easy 
to see that the range hypergraph of $I$ is $H$ and that $J$ is satisfiable if and only if so is $I$.
\end{proof}

\section*{Appendix B: Proofs of section \ref{sec:applications}}

\subsection*{Basic concepts and results about network hypergraphs}

If $T$ is an oriented tree and $x\in V(T)$ we denote by $T^{-}(x)$ (respectively $T^{+}(x)$) the subtree of 
$T$ containing all those nodes appear in some directed path ending (respectively, starting) at $x$.

A {\em rooted} tree is an oriented tree that is obtained from un undirected tree by fixing a node $r$, called the root,
and orienting all the edges away from the root or all the edges towards the root.

\begin{lemma}
\label{le:rootedtree}
For every hypergraph $H$ the following are equivalent:
\begin{enumerate}
\item Every pair of hyperedges in $H$ are either disjoint or contained in each other.
\item There exists a rooted tree $T$ that defines $H$ such that the path associated to every variable 
ends at the root.
\end{enumerate}
Furthermore, when condition (2) holds we can assume that in every edge $e=(x,y)$
associated to a minimal hyperedge of $H$, $x$ has in-degree zero.
\end{lemma}
\begin{proof}
$(1\Rightarrow 2)$. Let $H$ be an hypergraph satisfying (1). Let $T$ be the oriented tree defined 
in the following way. The node-set of $T$ is $E(H)\cup\{r\}$ where $r$ is a fresh node. For every
$h\in H(T)$, $T$ contains an edge from $h$ to $h'$ where $h'=r$ if $h$ is not contained in any other
hyperedge of $H$ and $h'$ is the smallest hyperedge in $H$ containing $h$ otherwise. Note that $T$
has root $r$ and that if $e=(x,y)$ is an edge
associated to a minimal hyperedge of $H$, then $x$ has in-degree zero. It is not difficult
to verify that $T$ satisfies condition (2).
$(2\Rightarrow 1)$. Assume that $T$ is a tree satisfying condition (2). Let $h$ and $h'$ be hyperedges
in $H$ and let $e_h=(x,y)$ and $e_{h'}=(x',y')$ be their associated edges in $T$. Consider three
cases: (a) there is a directed path from $x$ to $x'$, (b) there is a directed path from $x'$ to $x$,
(c) none of the previous holds. In case (a) it follows from the fact that the path associated to
every variable ends at the root that every path containing $e_h$ contains also $e_{h'}$ and hence
$h\subseteq h'$. By the same argument, case (b) implies that $h'\subseteq h$. Finally, in case (c)
there is no directed path containing both $e_h$ and $e_{h'}$ and, consequently, $h\cap h'=\emptyset$.
\end{proof}

Oriented trees can be combined by gluing some of their nodes. Formally, let $T_1$ and $T_2$ be trees and let $r_1$ and $r_2$ be nodes in $T_1$
and $T_2$ respectively. Assume, renaming nodes if necessary, that $V(T_1)\cap V(T_2)=\emptyset$. Then, the result of gluing $r_1$ and $r_2$
is obtained by, first, computing the disjoint union of $T_1$ and $T_2$ and, then, merging $r_1$ and $r_2$ into a new node $w$ that has as in-neighbours
the union of all in-neighbours of $r_1$ and $r_2$ and, as out-neighbours, the union of all out-neighbours of $r_1$ and $r_2$. It is not difficult to
see that the result is again an oriented tree.

The following technical lemma, which follows directly from the definitions, will be useful.

\begin{lemma}
\label{le:glue}
For $i=1,2$, let $H_i$ be an hypergraph, let $T_i$ be a tree defining $H_i$ and let $r_i\in V(T_i)$. Assume that $V(H_1)=V(H_2)$ and
that for every $v\in V(H_1)$ the following holds: if $v\in h_1\cap h_2$ with $h_1\in H_1$
and $h_2\in H_2$ then the last node of the path associated to $v$ in $T_1$ is $r_1$ and the first node of the path 
associated to $v$ in $T_2$ is $r_2$.  Then, the result of gluing $r_1$ 
and $r_2$ defines $H_1\cup H_2$.
\end{lemma}
\begin{proof}
Straightforward.
\end{proof}


\begin{lemma}
\label{le:helly}
Let $H$ be an network hypergraph defined by tree $T$ and let $J$ be a subhypergraph of $H$ with the property
that for every $a,b\in E(J)$ there exists some $c\in J$ with $a\cap c\neq\emptyset$ and $b\cap c\neq\emptyset$.
Then there exists an element $r\in V(T)$ such that for every $h\in E(J)$, $e_h\in E(T^{-}(r))\cup E(T^{+}(r))$
\end{lemma}
\begin{proof}
For every $h\in H_1\cup H_2$, let $e_h=(x,y)$ be its associated edge in $T$ and let us define $T_h$ to be
the subtree of $T$ with node-set $V(T^{-}(x))\cup V((T^+(y))$.

It follows that $V(T_a)\cap V(T_b)\neq\emptyset$ for every $a,b\in J$. Indeed, let $c$ be the hyperedge in $J$
such that $a\cap c\neq\emptyset$ and $b\cap c\neq\emptyset$. Then both endpoints of the edge $e_c$ 
associated to $c$ belong to $V(T_a)\cap V(T_{b})$. 

Next we shall use the 2-Helly property of the subtrees of a tree (see for example \cite{Bala00}).

\begin{lemma}($2$-Helly property of subtrees)
Let $T_1,\dots,T_n$ be a collection of subtrees of an (undirected) tree $T$ such that
for every $1\leq i,j\leq n$, $V(T_i)\cap V(T_j)\neq\emptyset$. It follows that 
$\bigcap_{1\leq i\leq n} V(T_i)\neq\emptyset$.
\end{lemma}

Note that the $2$-Helly property stated deals with undirected trees (instead of oriented trees). However it easily
implies that the same property holds for oriented trees. Hence, it follows 
that $\bigcap_{h\in J} V(T_h)\neq\emptyset$. To complete the proof note that
any vertex $r$ in $\bigcap_{h\in J} V(T_h)$ satisfies the conditions of the Lemma.
\end{proof}


\begin{lemma}
\label{le:technical}
Let $I=(V,D,L,{\mathcal C})$ be a conjunction of among constraints with $|D|\geq 3$, let $H$ be the 
hypergraph of its canonical booleanization, let $H_1$ be a set of at least 2 hyperedges in $E(H)$ where every
hyperedge, $S\times R$, in $H_1$ satisfies $S=V$ and let $H_2$ to be the subset of $E(H)$ containing, for every $v\in V$, 
the hyperedge $\{v\}\times D$. Then, if $H$ is a network hypergraph then there exists a tree $T$ defining $H$ and a node $r\in V(T)$
such that for every $h\in H_1$, $e_h\in E(T^{-}(r))$ and for every $h\in H_2$, $e_h\in E(T^+(r))$.
\end{lemma}
\begin{proof}
Assume that $H$ is a network hypergraph and let $T$ be a tree defining it. We shall use the following claim.
\begin{myclaim}
\label{cl:twozeros}
Let $x_1,e_1,x_2,\dots,e_{n-1},x_n$ be a directed path in $T$, and let
$e_i,e_j,e_k$, $i<j<k$ be different edges such two of their associated hyperedges
belong to $H_2$ and the remaining one to $H_1$. Then, the edge whose
associated hyperedge belongs to $H_1$ is $e_j$.
\end{myclaim}
\begin{proof}
Let $h_i,h_j,h_k$ be the hyperedges associated to $e_i$, $e_j$, and $e_k$ respectively and assume,
towards a contradiction, that $h_i\in H_1$ (the case $h_k\in H_1$ is symmetric).  Let $(v,d)\in h_i\cap h_k$
and let $p$ be its associated path. Clearly, $p$, contains both $e_i$ and $e_k$ and hence
it must contain also $e_j$ in contradiction with the fact that $h_j\cap h_k=\emptyset$
\end{proof}

For every directed path $p=x_1,\dots,x_n$ define $T(p)$ to be the subtree of $T$
induced by $V(T^{-}(x_1))\cup\{x_1,\dots,x_n\}\cup V(T^+(x_n))$. We define
$J^{-}$ (resp. $J^+$) to be the set containing all $h\in H_1\cup H_2$
with $e_h\in E(T^{-}(x_1))$ (resp. $e_h\in E(T^{+}(x_n))$).

The proof proceeds by showing that there exists a tree $T$ defining $H$ and a directed path $p$ in $T$
satisfying all the following properties:
\begin{enumerate}
\item $e_h\in E(T(p))$ for every $h\in H_1\cup H_2$. 
\item $J^{-}\cap H_1=\emptyset$ or $J^{-}\cap H_2=\emptyset$. Also, $J^{+}\cap H_1=\emptyset$ or $J^{+}\cap H_2=\emptyset$.
\item $J^{-}$ and $J^{+}$ are non empty.
\item $J^{+}=H_2$.
\end{enumerate}
Note that if $T$ and $p$ satisfy $(1)-(4)$ then $T$ and $r=x_n$ satisfy the Lemma.

Let us show the existence of $T$ and $p$ satisfying the above-mentioned properties in increasing order:

$(1)$. First, note that $J=H_1\cup H_2$ satisfies the hypothesis of 
Lemma \ref{le:helly}  and, hence, it follows that there exists an element $r$ in $T$ such that for 
every $h\in H_1\cup H_2$, $e_h$ belongs to $T^{-}(r)$ or $T^{+}(r)$. Hence, the path $p$ consisting of single node $r$ satisfies $(1)$.

$(1)\rightarrow (2)$. We claim that every path
$p=x_1,\dots,x_n$ of maximal length satisfying $(1)$ must satisfy $(2)$ as well. Assume, towards a contradiction, that $a_i\in J^{-}\cap H_i$ for $i=1,2$
(the proof for $J^+$ is analogous).
For every $h\in J^{-}$ let $q_h$ be a directed path in $T$ containing 
$e_h$ and ending at $x_1$. It follows that the last edge in $q_{a_1}$ and $q_{a_2}$ must be
identical since otherwise there could not be a directed path containing both $e_{a_1}$ and $e_{a_2}$, which would imply
that $a_1\cap a_2=\emptyset$. Let $(y,x_1)$ 
be the common last edge of $q_{a}$ and $q_{b}$. Then,
applying the same argument it follows that $(y,x_1)$ is also the last edge of $q_h$ for every $h\in J^{-}$.
Hence, by adding $y$ at the beginning of $y,x_1,\dots,x_n$ we obtain another path $p$ satisfying $(1)$
contradicting the maximality of $p$. 

$(1,2)\rightarrow(3)$. It is easy to see that every path $p$ satisfying $(1)$ and $(2)$ has a subpath that
satisfies, additionally, $(3)$.

$(1,2,3)\rightarrow(4)$. It follows from Claim \ref{cl:twozeros} that $J^{-}\cap H_2\neq\emptyset$ or 
$J^{+}\cap H_2\neq\emptyset$. We can assume, by reversing the direction of the edges in $T$ if necessary,
that $J^{+}\cap H_2\neq\emptyset$. Since, by (2), $J^+\cap H_1=\emptyset$, in order to show $(4)$ it is only necessary to 
prove that $H_2\subseteq J^+$. Assume towards a contradiction that there exists $h\in H_2\setminus J^+$.
It follows again by Claim \ref{cl:twozeros} and the fact that $H_1$ is nonempty that 
$h\in J^-$ and hence that $H_1\cap J^-=\emptyset$. 
Let $a=V\times A$, $b=V\times B$ be two different hyperedges in $H_1$. It follows 
that $e_a$ and $e_b$ appear in path $p$ and we can assume wlog. that $e_a$ appears before $e_b$ in $p$.
We shall prove that $B\subseteq A$.
 Let $d\in B$ and let $v\in V$ be such that $h=\{v\}\times D$. Consider the directed path $q$ associated to $(v,d)$. Clearly $q$ contains $e_h$
and $e_b$ and hence $q$ contains $e_a$ as well. Hence $(v,d)$ belongs to $(V,A)$, and hence, $d\in A$.
A similar reasoning, using now any arbitrary hyperedge in $J^{+}$ shows that $A\subseteq B$. Hence we 
have $A=B$, a contradiction. This completes the proof of (4).

\end{proof}

\subsection*{Proof of Theorem \ref{the:fulldomain}}
Let $H$ be the hypergraph of the canonical booleanization of $I$. 
Define $H_1$ to
be the set of all hyperedges $S\times R$ in $H$ with $S=V$.  Define $H_2$ to be 
the hypergraph associated to the non-empty assignment constraints. That is, $H_2$ contains for every $v\in V$, 
the hyperedge $\{v\}\times D$. Note that $h_1\cap h_2\neq\emptyset$ for every $h_1\in H_1$ and $h_2\in H_2$.

$(1)\Rightarrow (2)$. 
In what follows we shall assume that $|H_1|>1$ since otherwise
(2) follows directly. Assume that $H$ is a network hypergraph. It follows that $H_1$ and $H_2$ 
satisfy the hypothesis of Lemma \ref{le:technical}. Then, let $T$ the tree defining $H$ and 
$r$ the node in $T$ given by Lemma \ref{le:technical}.
Note that for every node occurring in any hyperedge in $H_1$, 
its associated path $p$ must contain some edge in $T^+(r)$ and, hence, must necessarily 
contain $r$ as well. Then, condition (2) follows by applying Lemma \ref{le:rootedtree} to $T^{-}(r)$.

$(2)\Rightarrow(1)$. We note that this is a particular case of direction $(2)\Rightarrow(1)$ in
Theorem \ref{the:addtogcc}. Still we include a proof since we think it might help the reader 
to understand the basic idea before embarking in the more complicated proof in \ref{the:addtogcc}.
This direction follows easily from Lemma \ref{le:rootedtree} and Lemma \ref{le:glue}. Assume that (2)
holds. Then it follows from Lemma \ref{le:rootedtree} that there exists a rooted tree $T_1$ defining $H_1$ such that, additionally,
the path associated to every variable in $V(H)$ ends at the root, $r_1$, of $T_1$. Also, note that
every two hyperedges in $H_2$ have an empty intersection which implies again by Lemma \ref{le:rootedtree}
that there exists a rooted tree $T_2$ defining $H_2$ such that the path associated to every variable $v\in V(H)$ starts at the root, $r_2$, of $T_2$.
Then, it follows from Lemma \ref{le:glue} that by gluing $r_1$ and $r_2$ in $T_1$ and $T_2$ we obtain the
tree defining $H$.

\subsection*{Proof of Theorem \ref{the:addtogcc}}

Let $H$ be the hypergraph of the canonical booleanization of $I$.

$(1)\Rightarrow (2)$. Define $H_1$ 
to be the set containing, for every $d\in D$, the hyperedge $V\times\{d\}$. Every such hyperedge is in $H$ because
we are assuming that the instance contains a global cardinality constraint. Furthermore, define $H_2$ to
be the set containing all non-empty assignment constraints. That is $H_2$ contains for every $v\in V$, 
the hyperedge $\{v\}\times D$. Notice that if $a=V\times\{d\}$ belongs to $H_1$ and
$b=\{v\}\times D$ belongs to $H_2$ then $a\cap b$ is non empty as it contains $(v,d)$. Assume
that $H$ is a network hypergraph. It follows
that $H_1$ and $H_2$ satisfy the hypothesis of Lemma \ref{le:technical} and let $T$ be the tree defining $H$
and $r$ the vertex in $V(T)$ given by Lemma \ref{le:technical}.

We claim that for every hyperedge $h\in H$, $e_h$ belongs to $E(T^{-}(r))\cup E(T^{-}(r))$
Indeed, if it is not the case
then there exists some hyperedge $h=S\times R$ such that its associated edge $e_h$ has one endpoint $y$ in $V(T^{-}(r))\cup V(T^{+}(r))$
and the other, $x$, outside. Assume that $y\in V(T^{-}(r))$ (the case $y\in V(T^{+}(r))$ is symmetric). It follows
that $e_h=(y,x)$ since otherwise $e_h$ would be included in $E(T^{-}(r))$. Let $(v,d)\in S\times R$ and let $p$ be the direct path in $T$
associated to $(v,d)$. This path must include $e_h$ and also the edge associated to $V\times\{d\}$, but this is impossible
since both edges must appear with different polarity because the edge associated to $V\times\{d\}$ belongs to $V(T^{+}(r))$.

For every $d\in D$, let $(x_d,y_d)$ be the edge in $T$ associated to hyperedge $V\times \{d\}$ and let $T_d$ be
$T^{-}(x_d)$. Also,
define $T_1$ to be the subtree of $T$ that contains all the nodes $x$ that belong to a path starting at $x_d$ for
some $d\in D$ and ending at $r$. 

For every $v\in V$, let $(x_v,y_v)$ be the edge in $T$ associated to hyperedge $\{v\}\times D$ and let $T_v$ be $T^{+}(y_v)$. Also,
define $T_2$ to be the subtree of $T$ that contains all the nodes $x$ that belong to a path starting at $r$ and
ending at $y_v$ for some $v\in D$.

Let $(v,d)\in V\times D$ and let $p$ be the directed path in $T$ associated to $(v,d)$. Clearly $p$ must 
contain the edges associated to $V\times\{d\}$ and $\{v\}\times D$ and, hence, also $r$.
It also follows that for every hyperedge, $h\in E(H)$, containing $(v,d)$, $e_h$ belongs
to $T_i$ with $i\in \{d,v,1,2\}$. It follows that every edge $e\in E(T)$ is contained in some
of the trees $T_i, i\in D\cup V\cup\{1,2\}$ we have just defined and hence 
we can infer the shape of $H$ by considering separately the shape of the
hypergraphs defined by each one of the trees. Then $(1)\Rightarrow(2)$ follows from the following lemma.
\begin{lemma}
Let $e$ (resp. $e'$) be an edge in $T$, let $h=S\times R$
(resp. $h'=S'\times R'$) be the hyperedge in $H$ associated to $e$ (resp. $e'$). Then the following holds:
\begin{enumerate}
\item If there exists $d\in D$ such that $e,e'\in E(T_d)$ then $R=R'=\{d\}$. Furthermore, $S$ and $S'$ are disjoint
or contained one in another.
\item If $e,e'\in E(T_1)$ then $S=S'=V$. Furthermore, $R$ and $R'$ are disjoint or contained one in another.
\item If there exists $v\in V$ such that $e,e'\in E(T_v)$ then $S=S'=\{v\}$. Furthermore, $R$ and $R'$ are disjoint
or contained one in another.
\item If $e,e'\in E(T_2)$ then $R=R'=D$. Furthermore, $S$ and $S'$ are disjoint or contained one in another.
\end{enumerate}
\end{lemma}
\begin{proof}
By symmetry we only need to prove the first two cases:
\begin{enumerate}
\item We first show that $R=\{d\}$ (the same argument shows that $R'=\{d\}$. Let $(v,d')$ be any node in $h$ and let $p$ be its associated directed path in $T$. Since
$p$ contains $r$, it must contain the edge associated to $V\times\{d\}$, which implies that $d'=d$. Hence, $R=\{d\}$.
The fact that $S$ and $S'$ are disjoint
or contained one in another follows from applying Lemma \ref{le:rootedtree} to $T_d$. Note that we use
that the path associated to every node in $V\times \{d\}$ must necessarily include $x_d$.
\item We first show that $S=V$ (the same argument shows that $S'=V$). Let $p$ be the directed path associated to any variable $(v,d)\in S\times R$.
This path contains $x_d$ and $r$ which implies that every path associated to a variable in 
$V\times\{d\}$ contains also $e_h$. It follows that $V\times\{d\}\subseteq S\times R$ and, hence, $S=V$.
 The fact
$S$ and $S'$ are disjoint
or contained one in another follows
by applying Lemma \ref{le:rootedtree} to $T_1$. Note that we use
that the associated path of every node in $V\times D$ must necessarily include $r$.
\end{enumerate}
\end{proof}
$(1)\Leftarrow(2)$. This follows easily from Lemma \ref{le:rootedtree} and Lemma \ref{le:glue}. 
For every $d$, let $H_d$ be the subhypergraph of $H$ containing
all the hyperedges $S\times R\in E(H)$ with $R=\{d\}$ and $S\neq V$. It follows from Lemma \ref{le:rootedtree} that there
is a rooted tree $T_d$ defining $H_d$ such that, additionally, all the paths associated to variables in $V(H)$
finish at the root of $T_d$, which we will denote by $r_d$. Similarly, let $H_1$ be the subhypergraph of $H$
containing all the hyperedges $S\times R\in E(H)$ with $S=V$. Again, by Lemma \ref{le:rootedtree} there
is a rooted tree $T_1$ with root, say, $r_1$ defining $H_1$. For every $d\in D$, let us denote by $(x_d,y_d)$ the edge in $T_1$ 
associated to $V\times \{d\}$. Lemma \ref{le:rootedtree} guarantees that $x_d$ has in-degree $0$ which implies
that the paths associated to every node in $V\times\{d\}$ start at $x_d$ and end at $r_1$. It follows
then from Lemma \ref{le:glue} that we can obtain a tree defining $H_1\cup \bigcup_{d\in D} H_d$ (that is,
the subgraph of $H$ containing all hyperedges $S\times R\in E(H)$ where $S=V$ or $R$ is a singleton)
by taking
the disjoint union of $T_1$ along with all the trees $T_d$ and gluing $r_d$ with $x_d$ for every $d\in D$.
The obtained tree, which we will call $T'_1$ is rooted at $r_1$ and has the additional property that every
path in $T'_1$ associated to a variable in $V(H)$ ends at the root $r_1$.
A symmetric argument shows that there exists a tree $T'_2$ with root $r_2$ that defines the subhypergraph of $H$
containing all hyperedges $S\times R\in E(H)$ where $R=D$ or $V$ is a singleton such
that, additionally, every path in $T'_1$ associated to a variable in $V(H)$ starts at the root $r_2$.
Then, it follows that by gluing $r_1$ and $r_2$ in $T'_1$ and $T'_2$ we obtain a 
tree defining $H$. This finishes the proof.

We note that direction $(1\Rightarrow 2)$ in Theorems \ref{the:fulldomain} and \ref{the:addtogcc} does not hold for boolean
domains due to the fact that the proof assumes the canonical encoding. However, the direction
$(2\Rightarrow 1)$ holds also for boolean domains.

\section*{Appendix C: Proof of the $O(n^{3/2}\log n\log \max)$ bound for the \sequence{} constraint}

Recall from Section \ref{sec:sequence} that a \sequence{} constraint is encoded as a boolean \cac{}
with constraints: $\among(\{x_i,\dots,x_{i+k}\},\{1\},\min,\max)$ where $i=1,\dots,n-k$ 
and $\min,\max,k$ are fixed integers.  Our goal is to improve the bound 
$O(n^{3/2}\log^2 n)$ given in Section \ref{sec:sequence} to $O(n^{3/2}\log n\log \max)$.To this end
we note, inspecting the proof of Lemma \ref{le:boolean}, that in the $\log mn$ factor appearing
in the function $f(n,m)$, $mn$ is has been obtained by bounding the quantity $u$ (where
$u$ is the maximum capacity of the edges in the network constructed by the algorithm)
by $O(\log n)$. We shall show that in the particular
case of the \sequence{} constraint we can obtain a better bound. In particular, 
we will see that $u\leq \max$ (and hence $\log u\leq \log\max$), which suffices to obtain
our desired bound. 

Recall the definitions of $a$, $b$, $c$, $T$,
and $P$ from the proof of Lemma \ref{le:boolean} and recall also that 
 the capacity of every edge in the network constructed by our algorithm is either an entry of vector $c$, which in the particular case of the \sequence{} constraint, is either $\max-\min$ or $1$ (and, hence, at most $\max$),  or the absolute value of an entry of vector $b=Pa$. In the particular case
of the \sequence{} constraint, all entries of $a$ are $\max$. Also, recall that the tree $T$ defining a the hypergraph
of a \sequence{} constraint is a directed path which implies that every row of $P$ has at most one $+1$, at most one $-1$, and the rest of entries are $0$. It follows that 
every entry in $b$ is in $\{-\max,0,\max\}$ and we are done.

\end{document}